\documentclass[12pt]{article}
\usepackage[a4paper]{geometry}
\usepackage[latin1]{inputenc}

\usepackage{xspace}
\usepackage{amsmath, amssymb, amsthm, dsfont, mathrsfs}
\usepackage[longnamesfirst]{natbib}
\usepackage{algorithmic}
\usepackage{algorithm2e}
\usepackage{url}
\usepackage{breakurl}

\newcommand{\wrt}{w.\,r.\,t.\xspace}

\newcommand{\ie}{i.\,e.\xspace}
\newcommand{\eg}{e.\,g.\xspace}

\newcommand{\ONEMAX}{\textsc{OneMax}\xspace}
\newcommand{\OneMax}{\ONEMAX}

\newcommand{\R}{\mathds{R}}

\DeclareMathOperator{\Prob}{Pr}

\newcommand{\E}[1]{\mathord{E}\mathord{\left(#1\right)}}

\newenvironment{proofof}[1]{\begin{proof}[Proof of #1]}{\end{proof}}

\newtheorem{theorem}{Theorem}

\newtheorem{lemma}{Lemma}

\title{The Fitness Level Method with Tail Bounds}

\author{Carsten Witt\\
\small DTU Compute\\
\small Technical University of Denmark\\
\small 2800 Kgs.\ Lyngby\\
\small Denmark}

\begin{document}
\maketitle

\begin{abstract}
The fitness-level method, also called the method of $f$\nobreakdash-based partitions, 
is an intuitive and widely used technique for the running time analysis of randomized search 
heuristics. It was originally defined to prove upper and lower bounds on the expected running time. Recently, 
upper tail bounds were added to the technique; however, these tail bounds only 
apply to running times that are at least twice as large as the expectation.

We remove this restriction and supplement the fitness-level method with sharp tail bounds, including lower tails. 
As an exemplary application, we prove that the running time of randomized local search on \OneMax 
is sharply concentrated around $n\ln n-0.1159n$.
\end{abstract}

\section{Introduction}
The running time analysis of randomized search heuristics, including evolutionary algorithms, ant 
colony optimization and particle swarm optimization, is a vivid research area where 
many results have been obtained in the last 15 years. Different methods for the analysis 
were developed as the research area grew. For an overview of the state of the art 
in the area see the books by \cite{AugerDoerrBook}, \cite{NeumannWittBook} and 
\cite{JansenBook}.

The fitness-level method, also called 
the method of fitness-based partitions, is a classical and intuitive method 
for running time analysis, 
first formalized by \citet{WegenerICALP01}. It applies to the case 
that the total running time of a search heuristic 
can be represented as (or bounded by) a 
sum of geometrically distributed waiting times, where the waiting times 
account for the number of steps spent on certain levels of the search space. 
\citet{WegenerICALP01} presented both upper and lower bounds on the running time 
of randomized search heuristics using 
the fitness-level method. The lower bounds relied on the assumption that 
no level was allowed to be skipped. \cite{SudholtTEC13} 
significantly relaxed this assumption and presented a very general 
lower-bound version of the fitness-level 
method that allows levels to be skipped with some probability.

Only recently, the focus in running time analysis turned to tail bounds, also 
called concentration inequalities. \cite{ZLLHProbable} were the first to add 
tail bounds to the fitness-level method. Roughly speaking, they prove \wrt\ the running time~$T$ 
that $\Prob(T>2\E{T}+2\delta h)=e^{-\delta}$ holds, where  $h$ is the worst-case expected waiting 
time over all fitness levels and $\delta>0$ is arbitrary. An obvious open question  
was whether the factor~$2$ in front of the expected value could be ``removed'' 
from the tail bound, \ie, replaced with~$1$; \cite{ZLLHProbable} only remark 
that the factor~$2$ can be replaced with $1.883$.

In this article, we give a positive answer to this question and supplement the fitness-level 
method also with lower tail bounds. Roughly speaking, we prove in Section~\ref{sec:tailbounds} that 
$\Prob(T<\E{T}+\delta )\le e^{-\delta^2/(2s)}$ and 
$\Prob(T>\E{T}+\delta )\le e^{-\min\{\delta^2/(4s), \delta h/4\}}$, where $s$ is the sum of the squares 
of the waiting times over all fitness levels. We apply the technique to  a classical 
benchmark problem, more precisely to the running time analysis of randomized local search (RLS)
on \OneMax in 
Section~\ref{sec:rls}, and prove a very sharp concentration of the running time around 
$n\ln n-0.1159n$. We finish with  some conclusions
and a pointer to related work.

\section{New Tail Bounds for Fitness Levels}
\label{sec:tailbounds}

\cite{Whatisatight} on the internet discussed tail bounds for a special case of our problem, namely the coupon 
collector problem \citep[][Chapter~3.6]{MotwaniRaghavan}. Inspired by this discussion, we present our main result in 
Theorem~\ref{theo:geometric-concentration} below. It applies to 
the scenario that a random variable (\eg, a running time) is given as a sum of 
geometrically distributed independent random variables (\eg, waiting times on fitness levels).
A concrete application 
will be presented in Section~\ref{sec:rls}.

\begin{theorem}
\label{theo:geometric-concentration}
Let $X_i$, $1\le i\le n$, be independent random variables following  
the geometric distribution 
with success probability $p_i$, and let $X:=\sum_{i=1}^n X_i$. 
If $\sum_{i=1}^n (1/p_i^2) \le s < \infty$ then for any 
$\delta>0$ 
\[
\Prob(X < \E{X} - \delta) \le e^{-\frac{\delta^2}{2s}}.
\]
For $h:=\min\{p_i\mid i=1,\dots,n\}$, 
\[
\Prob(X > \E{X} + \delta) \le e^{-\frac{\delta}{4}\cdot \min\left\{\frac{\delta}{s}, h\right\}}.
\]
\end{theorem}

For the proof, the following two simple inequalities will be used.

\begin{lemma}
\label{lem:inequality}
\mbox{}\\[-\bigskipamount]
\begin{enumerate}
\item For $x\ge 0$ it holds $\frac{e^x}{1+x}\le e^{x^2/2}$.
\item For $0\le x\le 1$ it holds $\frac{e^{-x}}{1-x} \le e^{x^2/(2-2x)}$.
\end{enumerate}
\end{lemma}

\begin{proof}
We start with the first inequality. The series representation of the exponential 
function yields
\[
e^{x} \;=\; \sum_{i=0}^{\infty } \frac{x^i}{i!} \;\le\; \sum_{i=0}^{\infty } (1+x)\frac{x^{2i}}{{(2i)}!}
\]
since $x\ge 0$. Hence, 
\[
\frac{e^{x}}{1+x} \;\le\; \sum_{i=0}^{\infty} \frac{x^{2i}}{(2i)!} .
\]
Since $(2i)! \ge 2^i i!$, we get 
\[
 \frac{e^{x}}{1+x} \;\le\;  \sum_{i=0}^{\infty } \frac{x^{2i}}{2^i i!} \;=\; e^{x^2/2}.
\]

To prove the second inequality, we omit all negative terms except for $-x$ 
from the series representation of $e^{-x}$ to get
\[
\frac{e^{-x}}{1-x} \;\le\; \frac{1 - x + \sum_{i=1}^{\infty } \frac{x^{2i}}{(2i)!}}{1-x} 
\;=\; 1+ \sum_{i=1}^{\infty } \frac{x^{2i}}{(1-x)\cdot (2i)!}.
\]
For comparison,
\[
e^{x^2/(2-2x)} \;=\; 1 + \sum_{i=1}^{\infty } \frac{x^{2i}}{2^i (1-x)^i i!} ,
\]
which, as $x\le 1$, is clearly not less than our estimate for $e^{-x}/(1-x)$. 
\end{proof}

\begin{proofof}{Theorem~\ref{theo:geometric-concentration}}
Both the lower and upper tail are analyzed similarly, using the exponential method (see, \eg, the proof of the Chernoff bound in
\citealp[][Chapter~$3.6$]{MotwaniRaghavan}).
We start with the lower tail. Let $d:=\E{X}-\delta = \sum_{i=1}^n (1/p_i) - \delta$. Since for any $t\ge 0$ 
\[
X < d \iff  -X > -d 
\iff 
e^{-tX} > e^{-td} ,\]
Markov's inequality and the independence of the $X_i$ yield that
\[
\Prob(X<d) \le \frac{\E{e^{-tX}}}{e^{-td} } = e^{td} \cdot \prod_{i=1}^n \E{e^{-tX_i}}.
\]
Note that the last product involves the moment-generating functions (mgf's) of the~$X_i$. 
Given a geometrically distributed random variable~$Y$ with parameter~$p$, its moment-generating 
function at $r\in\R$ equals 
$\E{e^{rY}} = \frac{pe^r}{1-e^r(1-p)} = \frac{1}{1 - (1-e^{-r})/p}$
for $r<-\ln(1-p)$. We will only use 
negative values for~$r$, which guarantees existence of the mgf's used in the following. Hence,
\[
\Prob(X<d) 
\le e^{td}  \cdot  \prod_{i=1}^n \frac{1}{1-(1-e^t)/p_i} \le e^{td} \cdot\prod_{i=1}^n \frac{1}{1+t/p_i},
\]
where we have used $e^x\ge 1+x$ for $x\in\R$. Now, by writing the numerators as $e^{t/p_i}\cdot e^{-t/p_i}$, using 
 $\frac{e^x}{1+x} \le e^{x^2/2}$ for $x\ge 0$ (Lemma~\ref{lem:inequality}) and finally 
plugging in $d$, 
we get 
\[
\Prob(X<d) 
\le e^{td} \cdot \left(\prod_{i=1}^n e^{t^2/(2p_i^2)} e^{-t/p_i}\right) = e^{td} e^{(t^2/2)\sum_{i=1}^n (1/p_i)^{2}}  e^{-t\E{X}} 
\le e^{-t\delta + (t^2/2) s}.
\]
The last exponent is minimized for $t=\delta/s$, which yields
\[
\Prob(X<d) \le e^{-\frac{\delta^2 }{2s}}
\]
and proves the lower tail inequality.

For the upper tail, we redefine $d:=\E{X}+\delta$ and obtain
\[
\Prob(X>d) \le \frac{\E{e^{tX}}}{e^{td} } = e^{-td} \cdot \prod_{i=1}^n \E{e^{tX_i}}.
\]
Estimating the moment-generating functions similarly as above, we get 
\[
\Prob(X>d) 
\le e^{-td} \cdot \left(\prod_{i=1}^n \frac{e^{-t/p_i}}{1-t/p_i} \cdot e^{t/p_i}\right).
\]
Since now positive arguments are used for the moment-generating functions, we limit 
ourselves to $t\le \min\{p_i\mid i=1,\dots,n\}/2=h/2$ to ensure convergence. 
Using $\frac{e^{-x}}{1-x} \le e^{x^2/(2-2x)}$ for $0\le x\le 1$ (Lemma~\ref{lem:inequality}), 
we get
\[
\Prob(X>d) 
\le e^{-td} \cdot \left(\prod_{i=1}^n  e^{t^2/(p_i^2 (2-2t/p_i))} \cdot e^{t/p_i}\right) = 
 \left(\prod_{i=1}^n  e^{-t\delta + t^2/p_i^2 }\right) \le e^{-t\delta+t^2s},
\]
which is minimized for $t=\delta/(2s)$. If $\delta\le s h$, this choice satisfies $t\le h/2$. Then 
$-t\delta +t^2 s=-\delta^2/(4s)$ and we get
\[
\Prob(X>d) \le e^{-\frac{\delta^2 }{4s}}.
\]
Otherwise, \ie\ if $\delta>sh$, we set 
$t=h/2$ to obtain $-t\delta+t^2 s = -\delta h/2+s (h/2)^2 \le -\delta h/2 + \delta h/4 = -\delta h /4$. 
Then 
\[
\Prob(X>d) \le e^{-\frac{\delta h}{4}}.
\]
Joining the two cases in a minimum leads to the lower tail.
\end{proofof}

Based on Theorem~\ref{theo:geometric-concentration}, we formulate the fitness-level theorem with tail bounds
for general optimization algorithms~$\mathcal{A}$ instead of a specific randomized search 
heuristic (see also \citealp{SudholtTEC13}, who uses 
a similar approach).
\begin{theorem}[Fitness Levels with Tail Bounds]
\label{theo:fitnesstail}
Consider an algorithm~$\mathcal{A}$ maximizing some function~$f$ and a partition of the search space into non-empty sets $A_1,\dots,A_m$. Assume 
that the sets form an $f$-based partition, \ie, 
for $1\le i<j\le m$  and all $x\in A_i$, $y\in A_j$ it holds $f(x)<f(y)$. 
We say that $\mathcal{A}$  is in $A_i$ or on level~$i$ if the best search point created so far 
is in $A_i$. 

\begin{enumerate}
\item 
If $p_i$ is a lower bound 
on the probability that a step of~$\mathcal{A}$ leads from level~$i$ to some higher level, independently 
of previous steps, then 
the first hitting time of~$A_m$, starting from level~$k$, is at most
\begin{align*}
\sum_{i=k}^{m-1} \frac{1}{p_i} + \delta.
\end{align*}
with probability at least $1-e^{-\frac{\delta}{4}\cdot \min\{\frac{\delta}{s}, h\}}$, for any finite 
$s\ge \sum_{i=k}^{m-1} \frac{1}{p_i^2}$ 
and $h = \min\{p_i \mid i=k,\dots,m-1\}$. 

\item 

If $p_i$ is an upper bound 
on the probability that a step of~$\mathcal{A}$ leads from level~$i$ to level~$i+1$, independently 
of previous steps, and 
the algorithm cannot increase its level by more than~$1$,  
then 
the first hitting time of~$A_m$, starting from level~$k$, is at least
\begin{align*}
\sum_{i=k}^{m-1} \frac{1}{p_i} - \delta
\end{align*}
with probability at least 
$1-e^{-\frac{\delta^2}{2s}}$.
\end{enumerate}
\end{theorem}

\begin{proof}
By definition, the algorithm cannot go down on fitness levels. 
Estimate the time to leave level~$i$ (from above resp.\ from below) by a geometrically distributed random variable with parameter~$p_i$ 
and apply Theorem~\ref{theo:geometric-concentration}.
\end{proof}

\section{Application to RLS on \OneMax}

\label{sec:rls}

We apply Theorem~\ref{theo:fitnesstail} to a classical benchmark problem in 
the analysis of randomized search heuristics, more precisely 
the running time of RLS on \OneMax. 
RLS is a well-studied randomized search heuristic, defined in 
Algorithm~\ref{alg:rls}. The function $\OneMax\colon\{0,1\}^n\to\R$ is defined by 
$\OneMax(x_1,\dots,x_n)=x_1+\dots+x_n$, and the running time 
is understood as the first hitting time of the all-ones string (plus~$1$ 
to count the initialization step).

\begin{algorithm}
\begin{algorithmic}
\STATE $t:=0$.
 \STATE choose an initial bit string $x_0 \in \{0,1\}^n$ uniformly at random.
 \STATE\textbf{repeat} 
  \STATE \quad create $x'$ by flipping a uniformly chosen bit in $x_t$.
  \STATE \quad $x_{t+1}:=x'$ if $f(x') \ge f(x_t)$, and $x_{t+1}:=x_t$ otherwise. 
  \STATE \quad $t:=t+1$.
 \STATE\textbf{forever.}
\end{algorithmic}
\mbox{}\\[-\smallskipamount]
\caption{RLS for the maximization of $f\colon\{0,1\}^n\to\R$}
\label{alg:rls}
\end{algorithm}

\begin{theorem}
\label{theo:rls-onemax}
Let $T$ be the running time of RLS on \OneMax. Then
\begin{enumerate}
\item
$n\ln n - 0.11594 n - o(n) \le E(T) \le  n\ln n - 0.11593n + o(n)$.
\item 
$\Prob(T\le E(T) - rn) \le e^{-\frac{3r^2}{\pi^2}}$ for any $r>0$.
\item
$\Prob(T\ge E(T) + rn) \le 
\begin{cases}
e^{-\frac{3r^2}{2\pi^2}} & \text{if $0<r\le \frac{\pi^2}{6}$}\\
e^{-\frac{r}{4}} & \text{otherwise}.
\end{cases}$.
\end{enumerate}
\end{theorem}

\begin{proof}
We start with  Statement~1, \ie, the bounds 
 on the expected running time. Let the fitness levels $A_0,\dots,A_n$ be defined 
by  
$A_i=\{x\in\{0,1\}^n\mid \OneMax(x)=i\}$ for $0\le i\le n$. By definition of RLS, the 
probability~$p_i$ of leaving level~$i$ equals $p_i=n/(n-i)$ for $0\le i\le n-1$. 
Therefore, the expected running time from starting level~$k$ is
\[
\sum_{i=k}^{n-1} \frac{n}{n-i} \;=\; n\sum_{i=1}^{n-k} \frac{1}{i},
\]
which leads to the weak upper bound $E(T)\le n\ln n + n$ in the first place. 
Due to the uniform initialization in RLS, Chernoff bounds yield
$\Prob(n/2-n^{2/3} \le k \le n/2+n^{2/3}) = 1-e^{-\Omega(n^{1/3})}$. We obtain 
\[
E(T) \;\le\; n\left(\sum_{i=1}^{n/2+n^{2/3}} \frac{1}{i}\right) + e^{-\Omega(n^{1/3})} \cdot (n\ln n+n) \;=\; 
n\left(\sum_{i=1}^{n/2+n^{2/3}} \frac{1}{i}\right) + o(n). 
\]
We can now estimate the Harmonic number by $\ln(n/2+n^{2/3})+\gamma + o(1) = \ln n + \gamma - \ln 2 + o(1)$, 
where $\gamma=0.57721\dots$ is the Euler-Mascheroni constant. Plugging in numerical values for $\gamma-\ln 2$ proves 
the upper bound on $E(T)$. The lower one is proven symmetrically.

For Statement~2, the lower tail bound, we use Theorem~\ref{theo:fitnesstail}. Now, 
$\sum_{i=1}^{n-k} \frac{1}{p_i^2 } \le \sum_{i=1}^{n} \frac{n^2}{i^2} \le \frac{n^2 \pi^2}{6} =: s$. Plugging $\delta:=rn$ 
and~$s$ in the second part of the theorem 
yields $\Prob(T\le E(T) -rn) \le e^{-\frac{r^2 n^2}{2s}} = 
e^{-\frac{3r^2}{\pi^2}} $.

For Statement~3, the upper tail bound, we argue similarly but have to determine when  $\frac{\delta}{s} \le h$. 
Note that $h=\min\{p_i\} = 1/n$. Hence, it suffices to determine when 
$\frac{6rn}{n^2 \pi ^2} \le 1/n$, which is equivalent to $r\le \frac{\pi^2}{6}$. Now the 
two cases of the lower bound follow by appropriately plugging $\frac{\delta}{s}$ or $h$ 
in the first part of Theorem~\ref{theo:fitnesstail}. 
\end{proof}

The stochastic process induced by RLS on \OneMax equals the classical and well-studied coupon collector 
problem (started with~$k$ full bins). Despite this fact, the lower tail bound 
from Theorem~\ref{theo:rls-onemax} 
could not be found in the literature (see also the comment introducing 
Theorem~1.24 in \citealp{DoerrTools}, which describes a simple but weaker 
lower tail). There is an easy-to-prove
 upper tail bound for the coupon collector of the kind 
$\Prob(T\ge E(T) + rn) \le e^{-r}$, which is stronger than our result 
but not obvious to generalize. Finally, \citet[][Theorem~3.38]{Scheideler2000} 
suggests upper and lower tail bounds for sums of geometrically distributed random variables, which 
could also be tried out in our example; however, it then turns out 
that these bounds are only useful if $r=\Omega(\sqrt{\ln n})$. 


\section{Conclusions}
We have supplemented upper and lower tail bounds to the fitness-level method. The lower tails are 
novel contributions and the upper tails improve an existing result from the literature significantly. 
As a proof of concept, we have applied the fitness levels with tail bounds to the analysis 
of RLS on \OneMax and obtained a very sharp concentration result. 

If the stochastic process under consideration is allowed to skip fitness levels, 
which is often the case with globally searching algorithms such 
as evolutionary algorithms, 
our upper tail bound may become arbitrarily loose and the lower tail is even 
unusable. To prove tail bounds in such cases, drift analysis may be used, which 
is another powerful and in fact somewhat related 
method for the running time analysis of randomized search heuristics. See, \eg, 
\cite{LehreWittarXiv13} and references therein for further reading.

\paragraph*{Acknowledgement.}
The author thanks Per Kristian Lehre for useful discussions.

\bibliographystyle{chicago}
\bibliography{fitness-levels-tails}
\end{document}